%% file: acl_latex.tex
\pdfoutput=1

\PassOptionsToPackage{table,xcdraw}{xcolor}
\PassOptionsToPackage{backref=page}{hyperref}
\documentclass[11pt]{article}

\usepackage[final]{acl}

\usepackage{times}
\usepackage{latexsym}

\usepackage[T1]{fontenc}

\usepackage[utf8]{inputenc}

\usepackage{microtype}
\newcommand\blfootnote[1]{
  \begingroup
  \renewcommand\thefootnote{}\footnote{#1}
  \addtocounter{footnote}{-1}
  \endgroup
}

\usepackage{inconsolata}

\usepackage{graphicx}

\usepackage{url}
\usepackage{booktabs}
\usepackage{amsmath,amsfonts,amssymb,amsthm}

\usepackage{xspace}
\usepackage[english]{babel}
\usepackage{xurl}
\usepackage{algorithm,algorithmic}

\usepackage{arydshln}
\usepackage{multirow}

\usepackage{tcolorbox}

\usepackage{pifont}

\tcbuselibrary{skins}
\tcbuselibrary{breakable}

\definecolor{my_green}{RGB}{0, 176, 80}
\definecolor{my_red}{RGB}{255, 0, 0}

\newcommand{\cmark}{\textcolor{my_green}{\ding{52}}}
\newcommand{\xmark}{\textcolor{my_red}{\ding{56}}}

\newtheorem{theorem}{Theorem}

\newtheorem{assumption}[theorem]{Assumption}

\newcommand{\method}{TSRL\xspace}

\title{\raisebox{-0.6ex}{\includegraphics[width=0.05\textwidth]{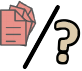}} \method: Trigger-based Supervised Reinforcement Learning for \\ Continual Agent Training}

\author{
 \textbf{Yuqian Fu\textsuperscript{1,2},}
 \textbf{Kaiwen Jiang\textsuperscript{1,2},}
 \textbf{Zijun Song\textsuperscript{1,2},}\\
 \textbf{Jiawei Wang\textsuperscript{3},}
 \textbf{Yuanheng Zhu\textsuperscript{1,2,\dag},}
 \textbf{Dongbin Zhao\textsuperscript{1,2}}
\\
 \textsuperscript{1}SKL-MAIS, Institute of Automation, Chinese Academy of Sciences,\\
 \textsuperscript{2}School of Artificial Intelligence, University of Chinese Academy of Sciences,\\
 \textsuperscript{3}University of Science and Technology of China
\\
 \normalsize{
   {\texttt{\{fuyuqian2022,yuanheng.zhu\}@ia.ac.cn}}
 }
}

\begin{document}
\maketitle

\blfootnote{$^\dagger$The corresponding authors.}

\begin{abstract}
Large language model (LLM) agents excel at complex, multi-step tasks but suffer from catastrophic forgetting during continual learning—acquiring new skills often degrades previously mastered capabilities. Existing post-training paradigms face a fundamental trade-off: Supervised Fine-Tuning (SFT) is sample-efficient but prone to severe forgetting, while Reinforcement Learning (RL) preserves prior capabilities but converges slowly without expert guidance. To bridge this gap, we introduce \textit{\textbf{T}rigger-based \textbf{S}upervised \textbf{R}einforcement \textbf{L}earning} (\method), a framework that integrates the strengths of both paradigms through an adaptive, state-dependent expert supervision mechanism. \method employs RL as the optimization backbone, activating SFT selectively on "hard states" where the policy fails. Theoretically, we show that restricting supervision to failure states derives an upper bound on the cumulative parameter updates, thereby reducing forgetting compared to full SFT. Empirically, \method achieves consistent gains across three agent benchmarks (ALFWorld, WebShop, and Deep Search) at both 7B and 32B scales. In continual learning scenarios, our method retains up to 9.6 points more capability on prior tasks than SFT$\to$RL baselines, while matching or exceeding target-task performance.
\end{abstract}

\section{Introduction}
\label{sec:intro}

Large language model (LLM) agents have demonstrated strong performance on complex, multi-step tasks across diverse domains~\citep{zhang2025landscape,xi2025survey,dong2025survey}.
However, deploying these agents in real-world environments presents a challenge: practical applications are inherently dynamic, requiring agents to continuously acquire new skills while retaining learned competencies~\citep{bell2025future, guo2025comprehensive}.
This \emph{continual learning} paradigm suffers from \emph{catastrophic forgetting}, a phenomenon where learning new tasks degrades performance on previously mastered ones~\citep{mccloskey1989catastrophic}.

\begin{figure}[t]
    \centering
    \includegraphics[width=\linewidth]{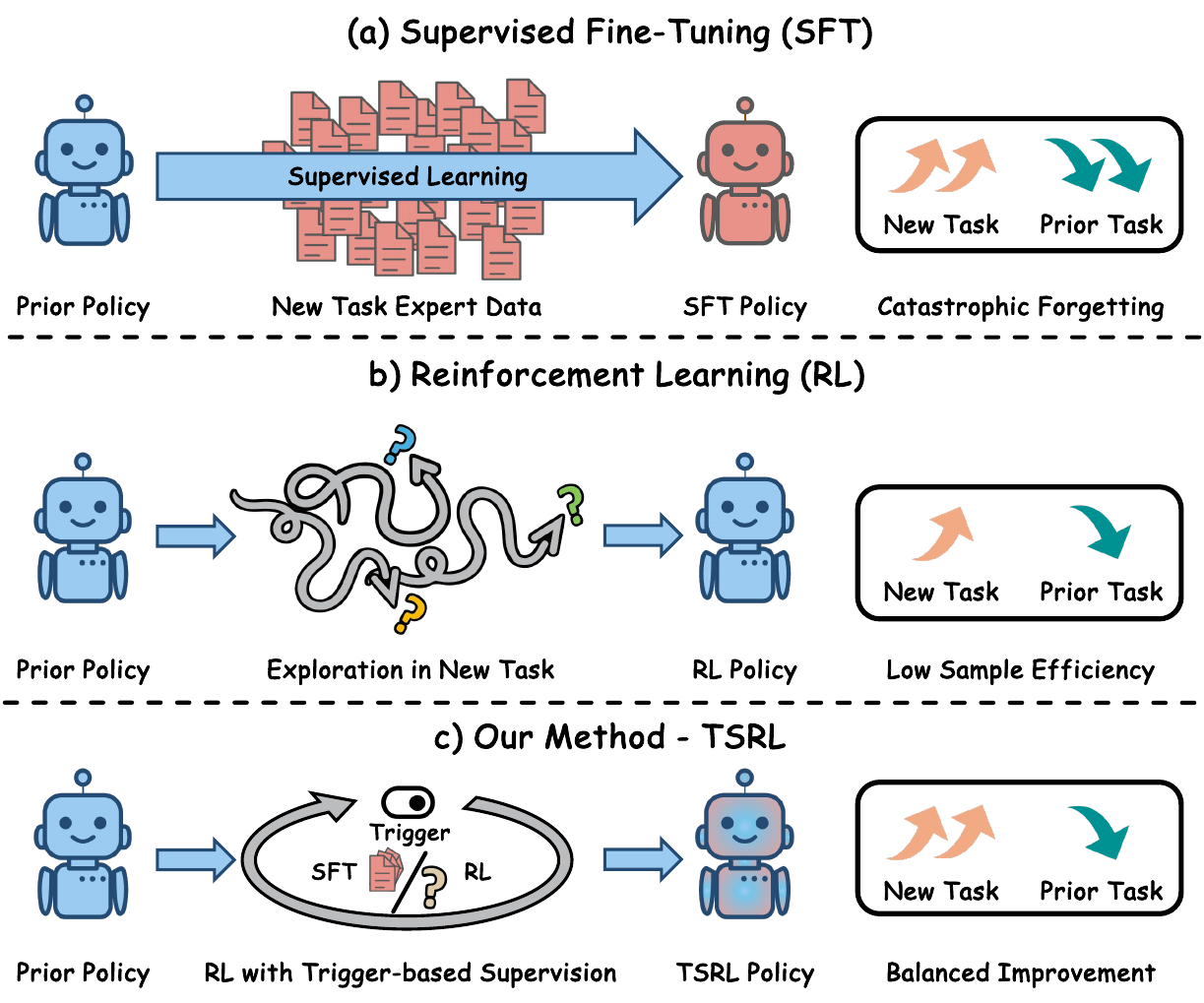}
    \vspace{-1.5em}
    \caption{\textbf{Overview} of different post-training methods for continual agent learning.}
    \label{fig:motivation}
    \vspace{-1.5em}
\end{figure}

\begin{figure*}[ht]
    \centering
    \includegraphics[width=1.0\textwidth]{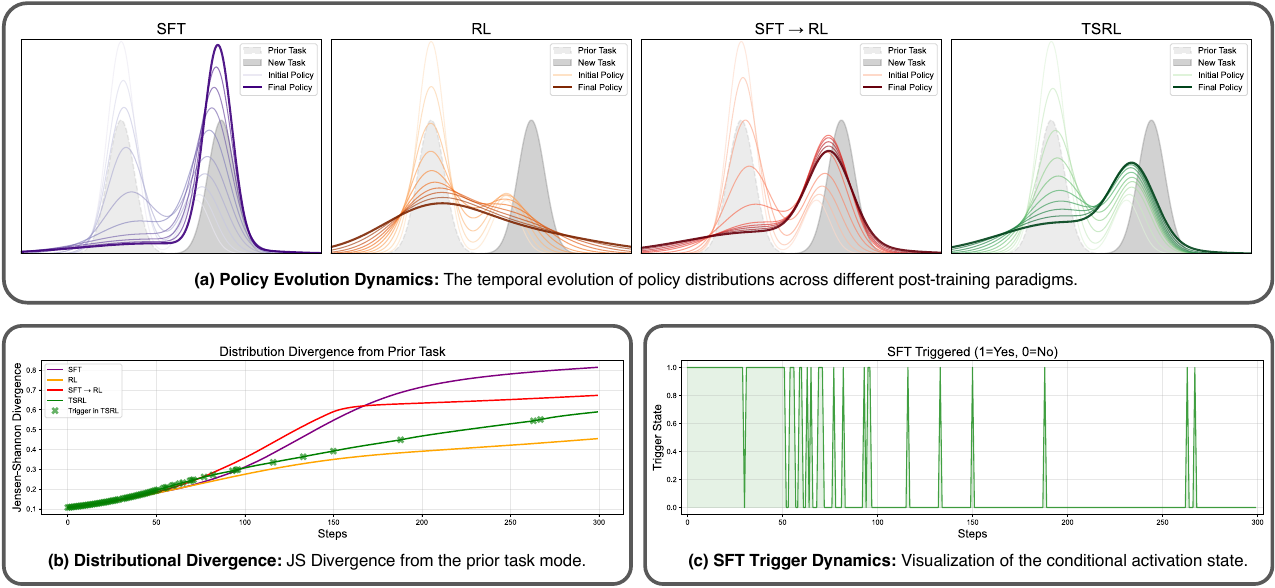}
    \vspace{-1.2em}
    \caption{\textbf{Illustrative case using bimodal Gaussian mixtures.} \method achieves a favorable trade-off between forgetting mitigation and effective utilization of expert data by leveraging trigger-based supervision.}
    \label{fig:case}
    \vspace{-1.5em}
\end{figure*}

Existing post-training approaches primarily employ two paradigms: Supervised Fine-Tuning (SFT) and Reinforcement Learning (RL)~\citep{tie2025large}.
Under continual learning, these paradigms show different behaviors: RL induces less forgetting than SFT while achieving comparable target-task performance~\citep{lai2025reinforcement,shenfeld2025rl}.
This performance discrepancy arises from their distinct optimization objectives.
SFT minimizes forward KL divergence, inducing \emph{mode-covering} behavior that forces the policy to match the entire support of the expert distribution.
In contrast, RL optimizes reverse KL divergence, promoting \emph{mode-seeking} behavior that concentrates on high-reward regions without overwriting other behavioral modes~\citep{korbak2022reinforcement,chen2025retaining}.
However, relying solely on RL ignores expert demonstrations, and is typically less sample-efficient: SFT provides dense, token-level supervision, whereas RL relies on sparse, response-level rewards, making credit assignment harder~\citep{fu2026revisiting,zhang2026rollouts}.
This motivates a critical question: \emph{How can we mitigate forgetting while effectively utilizing expert data?}

To address this question, we conduct an illustrative case study in Section~\ref{sec:example}, which reveals three key insights: RL exhibits minimal forgetting due to its mode-seeking nature; SFT induces severe forgetting through mode-covering behavior; and trigger-based supervision achieves a favorable trade-off.

We further formalize these observations theoretically. Under standard local approximations, we derive an upper bound on forgetting in terms of cumulative Fisher-weighted parameter updates---a quantity larger under full SFT than under selective supervision.
These findings motivate two design principles: \ding{182}~\emph{Sparsity via Locality}---supervision should be triggered only on hard states where the policy fails; and \ding{183}~\emph{On-policy Foundation}---the primary optimization should remain on-policy RL.

Guided by these principles, we propose \emph{Trigger-based Supervised Reinforcement Learning}~(\textbf{\method}), a principled framework that combines SFT and RL for continual agent training, illustrated in Figure~\ref{fig:motivation}.
The key insight behind \method is that expert demonstrations should be leveraged \emph{adaptively}---invoked only when the current policy struggles, rather than uniformly applied across all training instances.
Concretely, \method maintains RL as its optimization backbone and introduces a supervision trigger mechanism based on the Group Success Rate (GSR).
When the GSR falls below a predefined threshold---indicating that the policy lacks capability on a particular task---our method leverages the corresponding expert demonstration and injects a supervised gradient during training.
As training progresses and the policy improves, the trigger naturally deactivates for an increasing proportion of tasks, ensuring that supervision remains \emph{sparse} and \emph{localized} to genuinely difficult instances.
Additionally, we employ a temporal annealing schedule that gradually reduces the influence of expert supervision, allowing the agent to transcend the expert's performance limitations via RL exploration.

We evaluate \method on three agent benchmarks: ALFWorld and WebShop (with 4B and 7B models), and Deep Search (with 32B models).
Experimental results demonstrate that \method consistently outperforms baselines and exhibits resistance to forgetting in continual learning scenarios.

Overall, our \textbf{key contributions} are threefold:
\begin{itemize}
    \setlength{\itemsep}{0pt}
    \item We present an \textbf{illustrative case study} and \textbf{theoretical analysis} of the roles that SFT and RL play in continual learning, deriving two design guidelines: supervision should be sparse and localized to hard states, and the optimization backbone should remain on-policy RL.
    \item We propose \textbf{\method}, a framework that combines the sample efficiency of SFT with the forgetting-resistant properties of RL through a \textbf{trigger-based supervision} mechanism, enabling effective capability acquisition under the continual agent learning settings.
    \item We conduct extensive experiments demonstrating the effectiveness of \method across diverse agent benchmarks and model scales, including systematic evaluation of forgetting behavior in \textbf{continual learning} scenarios.
\end{itemize}

\section{Preliminaries}
\label{sec:preliminaries}

\paragraph{RL for LLM Agents.}
We formulate reinforcement learning for the LLM agent as a finite-horizon Markov Decision Process (MDP) $\mathcal{M} = (\mathcal{S}, \mathcal{A}, P, r, H)$~\citep{bellman1957markovian,puterman2014markov}.
Here, $\mathcal{S}$ denotes the state space, $\mathcal{A}$ the action space, $P: \mathcal{S} \times \mathcal{A} \to \Delta(\mathcal{S})$ the transition kernel (where $\Delta(\cdot)$ denotes the probability simplex), $r: \mathcal{S} \times \mathcal{A} \to \mathbb{R}$ the reward function, and $H$ the horizon (maximum number of interaction turns).
A policy $\pi: \mathcal{S} \to \Delta(\mathcal{A})$ maps states to distributions over actions and induces a trajectory distribution $\mathbb{P}_\pi$.
The RL objective is to maximize the expected cumulative reward: $J(\pi) = \mathbb{E}_{(s_t, a_t) \sim \mathbb{P}_\pi} [ \sum_{t=1}^{H} r(s_t, a_t) ]$.

In LLM agents, the state $s_t \in \mathcal{S}$ aggregates the interaction history up to turn $t$, including the initial prompt, previous actions (e.g., tool calls, reasoning traces), and environmental feedback (e.g., code outputs, search results). The action $a_t \in \mathcal{A}$ represents the agent's textual output, and the transition $P(s_{t+1} \mid s_t, a_t)$ captures environmental responses such as tool execution outcomes.

\paragraph{Problem Setting.}
In this work, we distinguish three policies: a pretrained \emph{prior policy} $\pi_0$ with parameters $\theta_0$, encoding prior capabilities; a \emph{task expert} $\pi_E$ optimized for the target task, accessible via offline demonstrations; and a \emph{training policy} $\pi_{t}$ in the new task at iteration $t$.

To quantify the degree of forgetting, let $\mathcal{D}_{\mathrm{prior}}$ denote a distribution on which we wish to preserve performance. We define $P_{\mathrm{prior}}(\pi)$ as the policy's performance (e.g., success rate or accuracy) on the prior task $\mathcal{D}_{\mathrm{prior}}$. The performance forgetting metric is defined as $P_{\mathrm{forget}} = P_{\mathrm{prior}}(\pi_{t}) - P_{\mathrm{prior}}(\pi_0)$.
During post-training, the aim is to achieve high new-task performance while minimizing the prior-task performance forgetting.

\paragraph{Baseline Post-Training Paradigms.}
In LLM post-training, three paradigms exhibit a fundamental trade-off. \textbf{Supervised Fine-Tuning (SFT)} minimizes $\mathbb{E}_{(s, a_E) \sim \mathcal{D}_E} [ -\log \pi(a_E \mid s) ]$ on expert data $\mathcal{D}_E \sim \mathbb{P}_{\pi_E}$. While sample-efficient, it risks overwriting capabilities in $\pi_0$. \textbf{Reinforcement Learning (RL)} directly optimizes $J(\pi)$ from $\pi_0$, preserving proximity to the prior policy but typically requiring substantially more interaction data because the learning signal is sparse and delayed.
A widely adopted LLM RL algorithm is Group Relative Policy Optimization (GRPO)~\citep{shao2024deepseekmath}, which eliminates the need for a value network. Specifically, for a given state $s$, GRPO samples a group of $G$ responses $\{a^{(i)}\}_{i=1}^G$ from the policy $\pi_{t}$, computes rewards $\{r^{(i)}\}_{i=1}^G$, and normalizes them to obtain group-relative advantages:
\begin{equation}
    \label{eq:group-relative-advantage}
        \hat{A}^{(i)} = \frac{r^{(i)} - \mathrm{mean}(\{r^{(j)}\}_{j=1}^G)}{\mathrm{std}(\{r^{(j)}\}_{j=1}^G)}.
\end{equation}
The policy is then updated via a clipped objective analogous to PPO~\citep{schulman2017proximal}, with an optional KL regularization term to prevent excessive deviation from a reference policy.
\textbf{SFT followed by RL} is a popular two-stage approach that first applies SFT and then continues with RL. Although effective for target-task performance, the initial SFT phase may cause forgetting that the subsequent RL phase cannot fully recover.

\section{Illustrative Case Study}
\label{sec:example}

To build intuition for how different post-training paradigms affect catastrophic forgetting, we conduct a controlled experiment using bimodal Gaussian mixture distributions, inspired by~\citet{chen2025retaining}.
As shown in Figure~\ref{fig:case}, we model the LLM agent's policy and target distributions for prior and new tasks as Gaussian mixtures.
This setup simulates a realistic, continual learning scenario where the new task partially overlaps with prior knowledge but requires acquiring novel behavioral patterns.
We compare four training paradigms over 300 optimization steps: pure SFT, pure RL, SFT$\to$RL (SFT for 150 steps, followed by RL for 150 steps), and our method \method.

The results reveal three key insights.
\ding{182}~\emph{Pure RL exhibits minimal forgetting} on the prior task (Figures~\ref{fig:case}(a) and (b)), as its mode-seeking nature concentrates probability mass on high-reward regions without perturbing other modes.
However, pure RL suffers from slow convergence due to the absence of expert guidance.
\ding{183}~\emph{SFT induces severe forgetting}: its mode-covering objective forces the policy to redistribute probability mass uniformly, overwriting parameters encoding prior competencies.
\ding{184}~\emph{\method achieves a favorable trade-off} between new-task acquisition and prior-task retention.
Figure~\ref{fig:case}(c) shows that the trigger rate decreases progressively as training proceeds, confirming that \method leverages expert data primarily in early learning when guidance is most beneficial, while relying on RL in later stages to mitigate forgetting.
In contrast, SFT$\to$RL applies dense supervision throughout the first 150 steps, incurring larger forgetting before transitioning to RL.
While the RL phase can partially recover lost capabilities, the initial damage from aggressive SFT updates is difficult to reverse fully.
\method's trigger mechanism ensures sparse supervision from the outset, preserving parameter subspaces critical for prior task performance.
Detailed experimental settings and analysis are provided in Appendix~\ref{sec:case_detail}.

\paragraph{Theoretical Justification.}
The empirical observations above can be formalized from a parameter-space perspective. We characterize forgetting via the Fisher Information Matrix (FIM) $F_{\text{prior}}$, which captures the sensitivity of the policy to parameter changes on the prior task distribution. Under a local quadratic approximation in Assumption~\ref{ass:quad}, forgetting in training is characterized by the Fisher-weighted displacement:
\begin{equation}
\Delta_{\text{forget}} \approx \frac{1}{2}\bigl\|{\textstyle\sum_t} \Delta\theta_t\bigr\|_{F_{\text{prior}}}^2.
\end{equation}
This expression can be further bounded in Theorem~\ref{thm:forgetting}. The resulting bound implies that forgetting is controlled by the cumulative magnitude of updates along directions sensitive to the prior task. In contrast, trigger-based supervision in \method restricts updates to hard states where the policy fails, yielding strictly smaller cumulative energy and thus reduced forgetting in continual learning. Full definitions and derivations appear in Appendix~\ref{sec:theory}.

\section{Methodology}
\label{sec:method}
The illustrative case study and theoretical analysis in Section~\ref{sec:example} establish that restricting expert supervision to sparse, hard states yields less forgetting than full SFT. We now present \textbf{T}rigger-based \textbf{S}upervised \textbf{R}einforcement \textbf{L}earning (\textbf{\method}), a framework that operationalizes these insights.

Our framework consists of three key components: \ding{182}~a group-based RL backbone that continuously optimizes the policy on the target task, \ding{183}~a failure detection mechanism that identifies hard states via group success rate statistics, and \ding{184}~an adaptive supervision module that injects expert supervision only when triggered. We describe each component in detail below.

\subsection{Trigger-based Supervision on Hard State}
\label{sec:failure}

To determine when expert supervision should be activated, we leverage the inherent exploration mechanism of group-based RL for failure detection.

For a given state $s$, the policy $\pi_\theta$ samples a group of $G$ actions $\{a_1, a_2, \dots, a_G\}$. Let $r(a_i|s) \in \{0, 1\}$ denote the binary reward for the $i$-th output. We define the \emph{Group Success Rate} (GSR) as:
\begin{equation}
    \label{eq:group-success-rate}
\mathrm{GSR}(s) = \frac{1}{G} \sum_{i=1}^G r(a_i|s).
\end{equation}
A state $s$ is classified as a \emph{hard state} if the policy fails to generate sufficient correct responses within the sampled group. Formally, the \emph{trigger} condition is defined by the indicator function:
\begin{equation}
    \mathbb{I}_{\mathrm{trigger}}(s) = \mathbb{I}\left( \mathrm{GSR}(s) \leq \delta \right),
    \label{eq:trigger}
\end{equation}
where $\delta \in [0, 1]$ is a predefined threshold. For instance, $\delta=0$ implies supervision is triggered when \emph{all} generations in the group fail, while $\delta=0.5$ activates supervision when fewer than half succeed.

This formulation offers two advantages. First, it dynamically adapts to the policy's evolving capability: as the agent learns and $\mathrm{GSR}(s)$ improves, the trigger automatically deactivates for an increasing proportion of states. Second, by conditioning on group-level statistics rather than individual samples, the detection mechanism exhibits robustness to stochastic variations in generation quality.

\subsection{Policy Optimization}
\label{sec:update}

At each iteration $t$, the optimization objective combines the RL objective with the SFT loss. We describe each component and its integration below.

\paragraph{RL Backbone.}
We employ group-based RL (e.g., GRPO~\citep{shao2024deepseekmath}, DAPO~\citep{yu2025dapo}) to estimate the policy gradient.
Specifically, following GRPO, for each output $a_i$ in the group, we compute the advantage $A_i$ by normalizing rewards within the group, as shown in Section~\ref{sec:preliminaries}. Following PPO~\citep{schulman2017proximal}, GRPO applies a clipped surrogate objective to stabilize training. The corresponding gradient is $g_t^{\mathrm{RL}} = \nabla_\theta \mathcal{L}_t^{\mathrm{RL}}$.

\paragraph{Adaptive Expert Supervision.}
When the trigger condition in Eq.~\eqref{eq:trigger} is satisfied ($\mathbb{I}_{\mathrm{trigger}}(s) = 1$), the algorithm queries the expert $\pi_E$ or retrieves the corresponding demonstration from an offline expert buffer to obtain a target action $a_E$. The SFT loss is computed only on on-policy states sampled from the current rollout buffer $\mathcal{B}_t$:
\begin{equation}
    \mathcal{L}_t^{\mathrm{SFT}} = \mathbb{E}_{s \sim \mathcal{B}_t} \left[
    \mathbb{I}_{\mathrm{trigger}}(s) \cdot \left( -\log \pi_{\theta_t}(a_E | s) \right) \right].
    \label{eq:sft_loss}
\end{equation}
The indicator function ensures that supervision remains sparse and localized to difficult instances, thereby minimizing the parameter drift associated with full SFT. Although the target action $a_E$ is off-policy, the trigger itself is evaluated on states encountered by the current policy, preserving the on-policy foundation of the overall training loop.

\paragraph{Combined Update with Annealing.}
The parameters are updated by combining both objectives:
\begin{equation}
    \theta_{t+1} = \theta_t - \eta \left( g_t^{\mathrm{RL}} + \lambda(t) \, g_t^{\mathrm{SFT}} \right),
    \label{eq:update}
\end{equation}
where $\eta$ is the learning rate, $g_t^{\mathrm{SFT}} = \nabla_\theta \mathcal{L}_t^{\mathrm{SFT}}$, and $\lambda(t)$ is a temporal annealing coefficient that linearly decays from $\lambda_0$ to $0$ over training.

The annealing schedule serves a critical purpose: in later training stages, the policy $\pi_{\theta}$ may develop specialized reasoning patterns that diverge from the expert $\pi_{E}$ yet yield high rewards. Enforcing strict alignment with the expert at this stage can introduce objective mismatch, hindering the policy's self-refinement. By annealing $\lambda(t) \to 0$, we allow the agent to transcend the expert's limitations via pure on-policy exploration once a sufficient competency foundation is established.

\subsection{Summary and Design Principles}
\label{sec:summary}

The complete training procedure is provided in Algorithm~\ref{alg:main}. We summarize the advantages of \method compared to existing paradigms in Table~\ref{tab:comparison}. Our algorithm directly instantiates two core design principles derived from the previous analysis:

\begin{itemize}
    \setlength{\itemsep}{0pt}
    \item \textbf{Sparsity via Locality.} By computing the SFT gradient strictly contingent on the trigger condition (Eq.~\eqref{eq:trigger}), we ensure that supervision is localized only to hard states.
    \item \textbf{On-policy Foundation.} The algorithm maintains the RL gradient as the constant driver of optimization. This ensures that policy updates remain grounded in the model's own generation probability landscape, preserving behavioral modes learned during pretraining.
\end{itemize}

\begin{table}[ht]
\centering
\caption{\textbf{Comparison of post-training paradigms.} \cmark~indicates the property is satisfied, and \xmark~indicates it is not. \method simultaneously achieves both high sample efficiency and forgetting resistance.}
\label{tab:comparison}
\small
\renewcommand{\arraystretch}{1}
\resizebox{\linewidth}{!}{
\begin{tabular}{>{}lcccc<{}}
\toprule
\textbf{Method} &\textbf{Sample Efficiency} & \textbf{Forgetting Resistance} \\
\midrule
Pure SFT       & \cmark & \xmark \\
Pure RL        & \xmark & \cmark\\
SFT$\to$RL     & \cmark & \xmark \\\midrule
\method (ours)  & \cmark & \cmark\\
\bottomrule
\end{tabular}}
\vspace{-1em}
\end{table}

\section{Experiments}
\label{sec:experiments}

We present comprehensive experiments to demonstrate the effectiveness of \method. Across three diverse benchmarks, our approach consistently outperforms baseline methods and exhibits resilience to forgetting in continual agent learning.

\subsection{Experimental Settings}

\paragraph{Datasets.}
Our evaluation spans three representative agent benchmarks targeting distinct task domains. \textbf{ALFWorld}~\citep{shridhar2021alfworld} focuses on embodied decision-making in household environments; \textbf{WebShop}~\citep{yao2022webshop} addresses interactive web navigation; and \textbf{Deep Search} involves complex multi-step information retrieval and synthesis, encompassing BrowseComp~\citep{wei2025browsecomp} and BrowseComp-ZH~\citep{zhou2025browsecomp}.
The reward structures differ across benchmarks: WebShop offers continuous rewards in the range $[0, 1]$, whereas ALFWorld and Deep Search provide sparse binary signals indicating task completion. All benchmarks are formulated as finite-horizon sequential decision problems. To ensure high-quality supervision, we retain only trajectories achieving a reward of 1 for constructing the training set. Detailed descriptions of the datasets are provided in Appendix~\ref{sec:data}. Appendix~\ref{sec:additional_experiments} further extends the continual-learning sequence with ScienceWorld~\citep{wang-etal-2022-scienceworld}.

\input{table/7b_result.tex}

\input{table/continual_7b.tex}

\paragraph{Training Setup.}
We instantiate LLM agents using Qwen2.5-7B-Instruct~\citep{yang2024qwen25} for ALFWorld and WebShop, and Qwen2.5-32B-Instruct for the more demanding Deep Search task. For all group-based RL algorithms, we set the rollout group size $G = 8$. The batch size is configured as $B = 64$ for 7B experiments and $B = 256$ for 32B experiments. The supervision trigger threshold $\delta$ is set to 0, indicating that expert demonstrations are incorporated only when all sampled outputs within a group fail to produce correct solutions. Experiments on ALFWorld and WebShop are executed on 8 NVIDIA H20 GPUs, while Deep Search experiments utilize 128 NVIDIA H20 GPUs. Complete hyperparameter configurations are documented in Appendix~\ref{sec:training_details}.

\paragraph{Baselines.}
We benchmark \method against several competitive approaches spanning two categories. First, we evaluate \textbf{closed-source LLMs} including GPT-4o~\citep{achiam2023gpt}, Claude-3.5-Sonnet~\citep{anthropic2024Claude35}, and Gemini-2.5-Pro~\citep{team2023gemini}, which serve as upper-bound references for powerful general-purpose agent capabilities. Second, we compare against standard \textbf{post-training paradigms}: pure SFT, pure RL, and the sequential combination SFT$\to$RL. Within the RL category, we consider two widely adopted algorithms: GRPO~\citep{shao2024deepseekmath} and DAPO~\citep{yu2025dapo}. Our proposed \method framework is instantiated as \method~$_{\text{GRPO}}$ and \method~$_{\text{DAPO}}$, applying the trigger-based supervision mechanism at each respective RL backbone. This design enables controlled comparisons of forgetting dynamics across different training paradigms. To further validate our approach, we additionally compare against more baselines  in Appendix~\ref{sec:additional_experiments}.

\subsection{Main Results}

Our comprehensive experiments demonstrate that \method yields performance improvements across a diverse range of tasks, baselines, and model scales, while exhibiting forgetting resistance.

\input{table/32b_result.tex}

\begin{figure*}[th]
    \centering
    \includegraphics[width=0.85\linewidth]{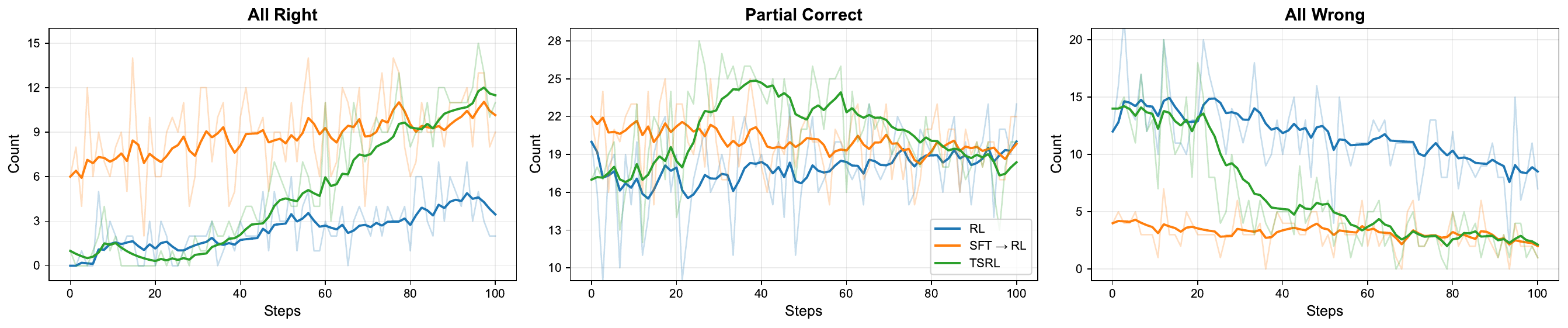}
    \vspace{-0.5em}
    \caption{\textbf{Dynamics of the success/failure ratio} within each training batch across optimization steps.}
    \label{fig:curves_all_groups_row}
    \vspace{-1em}
\end{figure*}

\paragraph{Performance on ALFWorld and WebShop.}
As shown in Table~\ref{tab:7b_result}, when training on their respective tasks, \method achieves consistent performance gains over all baselines on both benchmarks. Specifically, \method$_{\text{GRPO}}$ improves upon the SFT$\to$GRPO baseline by 4.3\% on ALFWorld and 3.5\% on WebShop (success rate), while \method$_{\text{DAPO}}$ yields improvements of 3.4\% and 4.4\%, respectively. These results confirm that trigger-based supervision effectively harnesses expert demonstrations while preserving the benefits of on-policy exploration.

\paragraph{Continual Learning on ALFWorld and WebShop.}
We evaluate the two-task continual learning setting where a model trained on one task is further adapted to the other. As shown in Table~\ref{tab:continual_7b}, in the WebShop direction (training on WebShop after ALFWorld), \method emerges as the \emph{sole} method that acquires strong new-task capability while preserving the prior task. In the reverse direction, pure RL methods preserve the prior task better than SFT-based baselines, but \method still offers the strongest overall balance between retention and final target-task performance. Appendix~\ref{sec:additional_experiments} shows that this trend also holds in a longer three-task sequence.

\paragraph{Performance on Deep Search.}
To further validate the effectiveness of our approach at larger model scales, we conduct experiments on the Deep Search task using Qwen2.5-32B-Instruct with GRPO. Since mathematical reasoning is among the model's core general capabilities, we use it as a proxy to measure forgetting. Specifically, we evaluate on the mathematical competition benchmarks AIME 24 and AIME 25~\citep{numina_math_datasets}, which require sophisticated multi-step reasoning.

As shown in Table~\ref{tab:32b_result}, our method achieves performance comparable to the SFT$\to$RL baseline on Deep Search while exhibiting significant advantages in preserving math reasoning capabilities. Notably, on AIME 25, the agent trained with \method even surpasses the base model, suggesting that our approach maintains its forgetting mitigation benefits at larger model scales. We also note that training dynamics may differ between 7B and 32B models, consistent with the observations in~\citet{tan2025scaling} and \citet{shenfeld2025rl}.
Appendix~\ref{sec:additional_experiments} further includes comparisons with representative hybrid baselines and experiments on Qwen3-4B, both of which support the same conclusion.

\subsection{Analysis}

To comprehensively evaluate \method, we conduct extensive analyses focusing on three critical aspects: the training dynamics across different methods, ablation studies on the trigger threshold, and the effect of the annealing mechanism.

\paragraph{Dynamics of the Success/Failure Ratio.}
To understand how different training paradigms affect policy competence throughout optimization, we track the success/failure ratio within each batch under the Deep Search setting. The results are presented in Figure~\ref{fig:curves_all_groups_row}.
Our findings indicate that RL methods suffer from relatively low sample efficiency. SFT provides dense token-level supervision, whereas RL only receives response-level rewards, making credit assignment harder. In contrast, \method, through its trigger-based SFT mechanism that selectively incorporates expert demonstrations, is able to converge more rapidly and surpass the performance of the SFT$\to$RL baseline.

\begin{figure}[t]
    \centering
    \includegraphics[width=\linewidth]{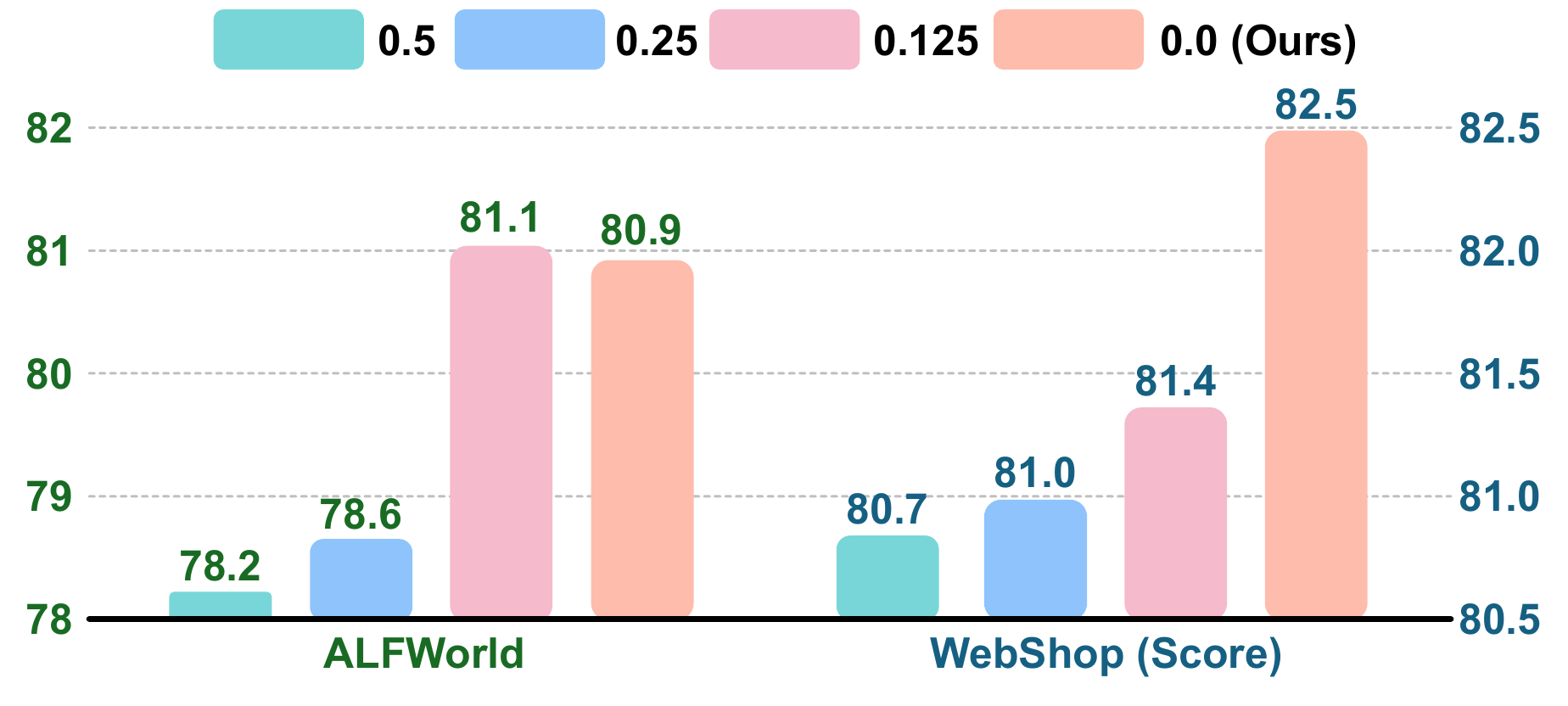}
    \vspace{-1.5em}
    \caption{\textbf{Performance comparison across different trigger threshold $\delta$} on ALFWorld and WebShop.}
    \vspace{-1.3em}
    \label{fig:threshold}
\end{figure}

\paragraph{Ablation Studies on Trigger Threshold.}
To evaluate the effect of the trigger threshold (GSR) on \method performance, we experiment with different threshold values $\delta$ on the ALFWorld and WebShop benchmarks. The results are shown in Figure~\ref{fig:threshold}. Our findings reveal that triggering SFT only when the entire group fails (i.e., $\delta = 0$) enables the most effective utilization of expert data. Higher threshold values introduce more frequent SFT updates, which negatively impact performance. This observation provides empirical support for our theoretical insight that applying SFT uniformly across all expert data during training is suboptimal and can impair subsequent RL performance.

\paragraph{Ablation Studies on Annealing Mechanism.}
To reduce interference between SFT and RL exploration in later training stages, \method incorporates a simple linear annealing schedule for the SFT loss weight. To assess the sensitivity of performance to annealing configurations, as shown in Table~\ref{tab:anneal-abla}, we compare the following settings: \ding{182}~Fixed weights of 1.0 and 0.1---both configurations result in noticeable performance degradation, confirming that persistent SFT supervision adversely affects final performance; \ding{183}~Annealing from 1 to 0 over 25 and 50 steps---the 50-step schedule achieves performance comparable to the 100-step configuration used in \method, indicating that performance is reasonably robust to the choice of annealing duration.
\section{Related Work}
\label{sec:related_work}

\paragraph{Reinforcement Learning for LLM Agents.}
Reinforcement learning has emerged as a strong paradigm for training LLM agents~\citep{ouyang2022training, guo2025deepseek, chai2025empowering}. Recent work mainly improves credit assignment, training stability, and uncertainty handling in agent RL~\citep{dong2025agentic, feng2025retool, wang2025harnessing, xu2025epo, wang2025ragen, zhang2025rlvmr}. Most of these methods focus on improving RL itself, whereas our goal is different: how to incorporate expert data into continual training without erasing prior capabilities. We defer the fuller survey to Appendix~\ref{app:related_work}.

\paragraph{Catastrophic Forgetting in Continual Training.}
A primary challenge in continual training is catastrophic forgetting, wherein acquiring new skills degrades previously learned capabilities~\citep{mccloskey1989catastrophic, kirkpatrick2017overcoming}. Recent evidence shows that RL forgets less than SFT because updates remain closer to the base policy and focus on high-reward modes instead of densely matching the expert distribution~\citep{lai2025reinforcement, shenfeld2025rl, chen2025retaining}. Hybrid methods such as UFT and LUFFY mix supervised and reinforcement signals~\citep{fu2026srft, liu2025uft, yan2025learning}, but they mainly target single-task reasoning rather than continual agent training. By contrast, \method uses sparse, outcome-based supervision specifically to control forgetting. Appendix~\ref{app:related_work} provides a more detailed comparison.

\input{table/anneal-ablation.tex}

\section{Conclusion}
In this paper, we introduce \method, a trigger-based supervised reinforcement learning framework for continual agent training that effectively combines the sample efficiency of SFT with the forgetting-resistant properties of RL. By maintaining RL as the optimization backbone and activating expert supervision only on hard states where the policy fails, \method adaptively leverages expert demonstrations without overwriting previously learned capabilities. Unlike existing post-training paradigms that apply supervision uniformly or sequentially, our trigger mechanism ensures that expert guidance remains sparse and localized to difficult tasks. Extensive experiments across three agent benchmarks demonstrate that \method consistently outperforms baseline methods while retaining more capability on prior tasks in continual learning scenarios.
This work represents a step toward enabling continual learning for LLM agents in real-world applications.

\clearpage

\section*{Limitations}
\label{sec:limitations}

The effectiveness of \method is closely tied to the quality of expert demonstrations and the design of the trigger-based supervision mechanism. In our experiments, we curate high-quality trajectories by retaining only those achieving maximum reward; however, in real-world scenarios where such filtering is impractical or expert policies are suboptimal, the supervision signal may introduce noise. Moreover, the current binary trigger condition based on group success rate is well-suited for tasks with sparse binary rewards, but may not fully capture the nuanced difficulty of states in tasks with dense or continuous reward structures.

Future work could explore robust aggregation strategies that leverage multiple expert sources of varying quality or develop mechanisms to automatically assess demonstration reliability. Extending the trigger mechanism to incorporate reward magnitude or uncertainty estimates would enable more fine-grained control over supervision activation. Another promising direction involves developing adaptive annealing strategies based on policy entropy or gradient statistics, which can dynamically adjust the transition from guided agent learning to autonomous exploration across different target tasks and training stages.

\section*{Acknowledgements}
This work was supported in part by Beijing Nova Program under Grant 20240484514 and in part by National Natural Science Foundation of China under Grants 62136008 and 62293541.

\bibliography{references}

\clearpage
\appendix

\renewcommand{\thefigure}{A\arabic{figure}}
\renewcommand{\thetable}{A\arabic{table}}
\renewcommand{\theHfigure}{A\arabic{figure}}
\renewcommand{\theHtable}{A\arabic{table}}

\setcounter{figure}{0}
\setcounter{table}{0}

\renewcommand{\theequation}{A\arabic{equation}}
\renewcommand{\theHequation}{A\arabic{equation}}
\setcounter{equation}{0}

\section*{Appendix}
\addcontentsline{toc}{section}{Appendix}

\section{Pseudocode of \method}

We describe the training process of \method framework in Algorithm~\ref{alg:main}.

\begin{algorithm}[h]
\caption{Training Process for \method}
\label{alg:main}
\baselineskip=1.1\baselineskip
\begin{algorithmic}
\STATE \textbf{Input:} Base policy $\pi_{0}$, expert policy $\pi_E$, training dataset $\mathcal{D}$, group size $G$, trigger threshold $\delta$, total training steps $T$
\STATE \textbf{Output:} Trained policy $\pi_{T}$
\STATE Initialize $\theta \leftarrow \theta_0$
\FOR{$t = 1, 2, \ldots, T$}
    \STATE Sample a batch of states $\{s_i\}_{i=1}^{B}$ from $\mathcal{D}$
    \STATE $\mathcal{L}_{\text{RL}} \leftarrow 0$, $\mathcal{L}_{\text{SFT}} \leftarrow 0$
    \FOR{each state $s$ in the batch}
        \STATE \textcolor[HTML]{0671b9}{\textbf{\texttt{// Group Success Rate Compute}}}
        \STATE Sample $G$ actions $\{a^{(j)}\}_{j=1}^{G} \sim \pi_{t}(\cdot \mid s)$
        \STATE Compute rewards $r$ for each action
        \STATE Compute group success rate via Eq.~\eqref{eq:group-success-rate}
        \STATE \textcolor[HTML]{0671b9}{\textbf{\texttt{// RL Gradient Compute}}}
        \STATE Compute the group relative advantages $\{\hat{A}^{(j)}\}_{j=1}^{G}$ via Eq.~\eqref{eq:group-relative-advantage}
        \STATE Accumulate RL loss via GRPO or DAPO
        \STATE \textcolor[HTML]{0671b9}{\textbf{\texttt{// Trigger-based SFT}}}
        \IF{$\text{GSR}(s) \leq \delta$}
        \STATE Query expert for actions $a_E \sim \pi_E(\cdot \mid s)$ or query offline expert demonstrations
        \STATE Accumulate SFT loss via Eq.~\eqref{eq:sft_loss}
        \ENDIF
    \ENDFOR
    \STATE \textcolor[HTML]{0671b9}{\textbf{\texttt{// Update Policy}}}
    \STATE Update the parameters of $\pi_{t}$ via Eq.~\eqref{eq:update}
\ENDFOR
\RETURN $\pi_{T}$
\end{algorithmic}
\end{algorithm}

\newpage
\section{Details of Illustrative Case}
\label{sec:case_detail}

This section provides detailed experimental settings and extended analysis for the illustrative case presented in Section~\ref{sec:example}.

\subsection{Experimental Setup}

We simulate the continual learning scenario using bimodal Gaussian mixture distributions, which capture the multi-modal structure commonly observed in LLM policy distributions.

\paragraph{Distribution Configuration.}
The prior-task distribution is represented by a mixture of two Gaussians:
\begin{equation}
    p_{\mathrm{prior}}(x) = 0.5 \cdot \mathcal{N}(x; \mu_1, \sigma^2) + 0.5 \cdot \mathcal{N}(x; \mu_2, \sigma^2),
\end{equation}
where $\mu_1$ and $\mu_2$ denote the two mode centers, and $\sigma^2$ is the shared variance.
The new-task distribution shares one mode with the prior task but introduces a novel mode:
\begin{equation}
    p_{\mathrm{new}}(x) = 0.5 \cdot \mathcal{N}(x; \mu_1, \sigma^2) + 0.5 \cdot \mathcal{N}(x; \mu_3, \sigma^2),
\end{equation}
where $\mu_3 \neq \mu_2$ represents the new mode location of the new tasks.
This configuration simulates a realistic scenario in which the new task partially overlaps with the prior task but requires acquiring novel behavioral patterns.

\paragraph{Training Paradigms.}
We compare four training paradigms over 300 optimization steps:
\begin{itemize}    \setlength{\itemsep}{0pt}
    \item \textbf{Pure SFT}: Supervised learning on expert demonstrations from the new task throughout all 300 steps.
    \item \textbf{Pure RL}: Policy gradient optimization without expert guidance for all 300 steps.
    \item \textbf{SFT$\to$RL}: SFT for steps 1--150, then RL for steps 151--300.
    \item \textbf{\method}: RL backbone with trigger-based SFT activated when the group success rate falls below threshold $\delta$.
\end{itemize}

\paragraph{Evaluation Metrics.}
We evaluate each paradigm using below two metrics:
\begin{itemize}   \setlength{\itemsep}{0pt}
    \item \textbf{Prior Task Retention}: The negative KL divergence between the learned policy and the prior task distribution, measuring how well the policy preserves prior capabilities.
    \item \textbf{New Task Acquisition}: The negative KL divergence between the learned policy and the new task distribution, measuring how effectively the policy learns the new task.
\end{itemize}

\subsection{Extended Analysis}

\paragraph{Mode-Seeking vs. Mode-Covering.}
The fundamental difference between RL and SFT lies in their optimization objectives.
SFT minimizes the forward KL divergence $D_{\mathrm{KL}}(\pi_E \| \pi_\theta)$, which encourages \emph{mode-covering} behavior: the policy must assign non-negligible probability mass wherever the expert distribution has support.
This forces the policy to spread out to cover all expert modes, inevitably overwriting the probability mass allocated to prior task modes.
In contrast, RL effectively minimizes the reverse KL divergence $D_{\mathrm{KL}}(\pi_\theta \| \pi_{\mathrm{reward}})$, promoting \emph{mode-seeking} behavior: the policy concentrates probability mass on high-reward regions without being forced to cover all modes.
This allows the policy to acquire new capabilities while leaving the modes corresponding to prior tasks relatively undisturbed.

\paragraph{Trigger Frequency Dynamics.}
As shown in Figure~\ref{fig:case}(c), the supervision trigger frequency in \method exhibits a characteristic decay pattern.
In early training, the policy frequently fails on novel queries, triggering expert supervision to provide guidance.
As the policy improves, it succeeds on an increasing proportion of queries, causing the trigger to deactivate.
This adaptive supervision behavior ensures that:
\begin{itemize}
    \setlength{\itemsep}{0pt}
    \item Expert knowledge is utilized when most beneficial (early training, hard instances).
    \item The policy transitions to pure RL exploration as competency develops.
    \item Cumulative supervised updates remain sparse, minimizing forgetting.
\end{itemize}

\paragraph{Comparison with SFT$\to$RL.}
While SFT$\to$RL also combines both paradigms, its sequential structure leads to suboptimal dynamics.
The dense SFT phase (steps 1--150) causes substantial parameter drift away from the base model, overwriting parameters critical for prior task performance.
Although the subsequent RL phase can partially recover some capabilities by re-exploring the policy space, the initial damage proves difficult to fully reverse.
In contrast, \method interleaves sparse SFT with continuous RL from the beginning.
The trigger mechanism ensures that even in early training, supervision is applied only to genuinely difficult instances where the policy cannot find correct solutions through exploration.
This preserves the parameter subspaces encoding prior capabilities while providing guidance for new task acquisition.

\paragraph{Theoretical Connection.}
These empirical observations directly support the theoretical analysis in Appendix~\ref{sec:theory}.
The Fisher information framework establishes that forgetting is bounded by the cumulative Fisher-weighted energy of parameter updates.
By restricting supervision to hard states (where the policy has a near-zero success rate), \method ensures that the cumulative supervision energy ($C_{\mathrm{tri.}}$) remains small compared to full SFT ($C_{\mathrm{ful.}}$).
The trigger mechanism thus provides a practical instantiation of the theoretical principle that sparse supervision minimizes destructive parameter drift.

\section{Theoretical Analysis}
\label{sec:theory}

The dichotomy between SFT and pure RL motivates a central question: \emph{How can we leverage expert knowledge to improve sample efficiency without catastrophically overwriting the prior capabilities?}
We posit that not all expert interventions are necessary; the utility of supervision is state-dependent.
This section develops a theoretical framework demonstrating that by restricting expert supervision exclusively to states where the policy fails, we can bound the parameter drift relative to SFT, thereby minimizing forgetting.

\subsection{Fisher Information Characterization of Forgetting}
\label{sec:fisher}

To quantify forgetting, we must measure how much the policy's behavior changes on the \textit{prior} distribution $\mathcal{D}_{\mathrm{prior}}$, even though training occurs on the \textit{new} task.
Directly measuring function space divergence is intractable.
Therefore, we resort to a local quadratic approximation, where the \emph{Fisher Information Matrix} (FIM)~\citep{kirkpatrick2017overcoming} serves as the Riemannian metric tensor, linking parameter drift to functional degradation.

Let $\theta_0$ denote the parameters of the base policy $\pi_0$. Define the prior preserve loss $L_{\mathrm{prior}}(\theta) \triangleq \mathbb{E}_{(s,a) \sim \mathcal{D}_{\mathrm{prior}}}[-\log \pi_\theta(a|s)]$.

\begin{assumption}[Local quadratic structure]
\label{ass:quad}
There exists a positive semidefinite matrix $F_{\mathrm{prior}}$, the Fisher information matrix of $\pi_0$ under $\mathcal{D}_{\mathrm{prior}}$, such that for parameters $\theta$ in a neighborhood of $\theta_0$:
\begin{equation}
    L_{\mathrm{prior}}(\theta) \approx L_{\mathrm{prior}}(\theta_0) + \frac{1}{2} (\theta - \theta_0)^\top F_{\mathrm{prior}} (\theta - \theta_0).
\end{equation}
\textit{Remark.} $\nabla L_{\mathrm{prior}}(\theta_0) = 0$ is satisfied when the base model $\pi_0$ has been pretrained to convergence on data drawn from $\mathcal{D}_{\mathrm{prior}}$. This is a standard assumption for foundation models in continual learning, which typically undergo extensive pretraining before task-specific fine-tuning.
We do not assume that the global LLM loss landscape is quadratic; the approximation is only used locally around the pretrained prior policy, which is the relevant regime for continual post-training initialized from a converged model.
\end{assumption}

Under Assumption~\ref{ass:quad}, the forgetting metric $\Delta_{\mathrm{forget}} \triangleq L_{\mathrm{prior}}(\pi_T) - L_{\mathrm{prior}}(\pi_0)$ admits the representation:
\begin{equation}
    \Delta_{\mathrm{forget}} \approx \frac{1}{2} \|\theta_T - \theta_0\|_{F_{\mathrm{prior}}}^2,
    \label{eq:forget_fisher}
\end{equation}
where $\|v\|_{F_{\mathrm{prior}}}^2 \triangleq v^\top F_{\mathrm{prior}} v = \|F_{\mathrm{prior}}^{1/2} v\|_2^2$ denotes the squared Fisher-weighted norm.

This formulation implies that to minimize forgetting, the optimization trajectory should avoid directions with high curvature on the prior manifold (i.e., directions where $F_{\mathrm{prior}}$ is large).

\subsection{Forgetting Reduction via Selective Supervision}
\label{sec:restricted}

We now analyze the impact of the post-training paradigm on parameter drift.
Although the gradient updates are computed based on the \textit{target} task, their accumulation results in a displacement vector $\theta_T - \theta_0$. We evaluate the cost of this displacement using the geometry of the prior task ($F_{\mathrm{prior}}$).

Let $\Delta\theta_t^{\mathrm{ful.}}$ denote parameter updates derived from standard SFT (full supervision) at iteration $t$, and $\Delta\theta_t^{\mathrm{tri.}}$ denote updates from trigger-based SFT (selective supervision).
We characterize their cumulative impacts via their Fisher-weighted energy.

\begin{assumption}[Cumulative update bounds]
\label{ass:sft-mags}
There exist nonnegative constants $C_{\mathrm{ful.}}$ and $C_{\mathrm{tri.}}$ such that the cumulative Fisher-weighted energy of the updates satisfies:
\begin{equation}
\begin{aligned}
    \textstyle\sum_t \|F_{\mathrm{prior}}^{1/2} \Delta\theta_t^{\mathrm{ful.}}\|^2 &= C_{\mathrm{ful.}}, \\
    \textstyle\sum_t \|F_{\mathrm{prior}}^{1/2} \Delta\theta_t^{\mathrm{tri.}}\|^2 &= C_{\mathrm{tri.}}.
\end{aligned}
\end{equation}
\end{assumption}

Crucially, $C_{\mathrm{tri.}}$ aggregates gradients only from the subset of states where the policy fails, whereas $C_{\mathrm{ful.}}$ aggregates gradients from the entire state distribution.
Consequently, assuming the policy is not failing universally, the sparsity of the trigger mechanism implies $C_{\mathrm{tri.}} \ll C_{\mathrm{ful.}}$.

\begin{assumption}[RL update bound]
\label{ass:rl-mags}
There exists a nonnegative constant $C_{\mathrm{RL}}$ representing the cumulative energy of the on-policy RL updates:
\begin{equation}
    \textstyle\sum_t \|F_{\mathrm{prior}}^{1/2} \Delta\theta_t^{\mathrm{RL}}\|^2 \le C_{\mathrm{RL}}.
\end{equation}
\end{assumption}

We formally compare two paradigms: \ding{182}~SFT followed by RL ($\theta_T^{\mathrm{fol.}}$), and \ding{183}~Trigger-based SFT interleaved with RL ($\theta_T^{\mathrm{int.}}$).

\begin{theorem}[Bound on Forgetting]
\label{thm:forgetting}
Under Assumptions~\ref{ass:quad}--\ref{ass:rl-mags}, forgetting for both SFT$\to$RL and trigger-based SFT$\to$RL is bounded by:
\begin{equation}
\begin{aligned}
    \Delta_{\mathrm{forget}}^{\mathrm{fol.}} &\lesssim \frac{1}{2} T\bigl( \sqrt{C_{\mathrm{RL}}} + \sqrt{C_{\mathrm{ful.}}} \bigr)^2, \\
    \Delta_{\mathrm{forget}}^{\mathrm{int.}} &\lesssim \frac{1}{2} T\bigl( \sqrt{C_{\mathrm{RL}}} + \sqrt{C_{\mathrm{tri.}}} \bigr)^2.
\end{aligned}
\end{equation}
\end{theorem}

\begin{proof}
The total parameter displacement equals the vector sum of RL updates and supervised updates. For the interleaved case, $\theta_T^{\mathrm{int.}} - \theta_0 = \sum_t \Delta\theta_t^{\mathrm{RL}} + \sum_t \Delta\theta_t^{\mathrm{tri.}}$.
Using the definition of the Fisher norm $\|v\|_{F_{\mathrm{prior}}} = \|F_{\mathrm{prior}}^{1/2}v\|_2$ and the Triangle Inequality:
\begin{equation}
\begin{aligned}
    \|\theta_T^{\mathrm{int.}}& - \theta_0\|_{F_{\mathrm{prior}}} \\
    &\le \|\sum_t \Delta\theta_t^{\mathrm{RL}}\|_{F_{\mathrm{prior}}} + \|\sum_t \Delta\theta_t^{\mathrm{tri.}}\|_{F_{\mathrm{prior}}}  \\
    &\le \sum_t \|\Delta\theta_t^{\mathrm{RL}}\|_{F_{\mathrm{prior}}} + \sum_t \|\Delta\theta_t^{\mathrm{tri.}}\|_{F_{\mathrm{prior}}}  \\
    &\le \sqrt{T \cdot C_{\mathrm{RL}}} + \sqrt{T \cdot C_{\mathrm{tri.}}}.
\end{aligned}
\end{equation}
\end{proof}

Theorem~\ref{thm:forgetting} indicates that forgetting is driven by the magnitude of the supervised signal. Since the trigger mechanism restricts supervision to a subset of data (hard states), it inherently ensures $C_{\mathrm{tri.}} \ll C_{\mathrm{ful.}}$, thereby theoretically guaranteeing reduced forgetting compared to full SFT.

\subsection{Design Principles}
\label{sec:principles}

Supported by the analysis above, we adhere to two core principles:

\paragraph{\ding{182} Sparsity via Locality.} Supervision should not be uniform. By triggering expert intervention only on hard states, we minimize the total energy of supervised updates ($C_{\mathrm{tri.}}$), which is the dominant factor in controlling forgetting.

\paragraph{\ding{183} On-policy Foundation.} The base optimization must remain on-policy RL. This ensures that the primary learning signal respects the model's current generation probability, whereas offline SFT forces the model towards an external distribution, often incurring a larger penalty in the Fisher metric.

\newpage
\section{Extended Related Work}
\label{app:related_work}

This appendix provides a comprehensive survey of recent advances in reinforcement learning for LLM agents, complementing the discussion in Section~\ref{sec:related_work}.

\subsection{Reinforcement Learning for LLM Agents}
Building upon Proximal Policy Optimization (PPO)~\citep{schulman2017proximal}, Group Relative Policy Optimization (GRPO)~\citep{shao2024deepseekmath} eliminates the need for a separate value network by estimating advantages through group-based normalization, significantly reducing computational overhead and becoming a standard algorithm for RL with verifiable rewards (RLVR). However, extending GRPO to long-horizon agent tasks faces critical challenges in credit assignment and training stability. Recent work has proposed various algorithmic innovations to address these issues.

\paragraph{Credit Assignment.}
GiGPO~\citep{feng2025gigpo} introduces a two-level hierarchical structure that computes macro advantages at the episode level and micro advantages at the step level via anchor state grouping, achieving fine-grained credit assignment without additional rollouts. RLVMR~\citep{zhang2025rlvmr} moves beyond outcome-only supervision by providing dense, verifiable rewards for meta-reasoning behaviors such as planning, exploration, and reflection through programmatic rules.

\paragraph{Entropy Management.}
ARPO~\citep{dong2025agentic} observes that LLMs exhibit high entropy following tool interactions and proposes an entropy-based adaptive rollout mechanism that balances global trajectory sampling and step-level branching to align tool-use behaviors. EMPG~\citep{wang2025harnessing} identifies that policy gradient magnitudes are inherently coupled with entropy, leading to inefficient updates, and recalibrates learning signals based on step-wise uncertainty to amplify confident correct actions while attenuating noisy exploration. EPO~\citep{xu2025epo} addresses the exploration-exploitation cascade failure unique to multi-turn sparse-reward settings through entropy smoothing regularization and adaptive phase-based weighting.

\paragraph{Training Stability.}
ReTool~\citep{feng2025retool} enables dynamic interleaving of real-time code execution within reasoning processes and trains models through outcome-based RL to learn when and how to invoke computational tools. RAGEN~\citep{wang2025ragen} proposes StarPO, a general framework for trajectory-level agent RL that identifies the ``Echo Trap'' phenomenon where agents collapse into repetitive templates, and introduces stabilization techniques including trajectory filtering and gradient normalization.

\subsection{Catastrophic Forgetting in Agent Training Process}

\subsubsection{SFT and RL's Forgetting Mechanisms}

Recent studies have identified several mechanisms explaining why RL naturally mitigates catastrophic forgetting compared to SFT. First, RL updates are scaled by reward variance, acting as a data-dependent regularizer that protects sensitive parameter subspaces~\citep{lai2025reinforcement}. Second, on-policy RL remains biased toward KL-minimal solutions close to the base policy, whereas SFT may converge to distributions arbitrarily far from initialization~\citep{shenfeld2025rl}. Third, the mode-seeking nature of RL preserves existing behavioral modes while concentrating probability mass on high-reward regions; in contrast, SFT's mode-covering objective redistributes mass uniformly, overwriting prior competencies~\citep{chen2025retaining}. Finally, SFT updates tend to be redundant and high-magnitude, which risks overwriting parsimonious RL-acquired knowledge~\citep{yuan2025mitigating}.

Building on these insights, hybrid strategies have been proposed to further stabilize continual learning. MIFO~\citep{yuan2025mitigating} freezes RL-critical parameters during subsequent training phases, while RIF-RFT~\citep{lai2025reinforcement} filters rollout instances to enhance both efficiency and knowledge preservation. \citet{wang2025nemotron} propose Cascade RL for multi-domain training, requiring explicit removal of overlapping prompts between stages to prevent interference. Iterative SFT-style methods, motivated by the role of approximately on-policy data in mitigating forgetting~\citep{chen2025retaining}, alternate between online data collection and supervised updates to improve harder examples. Although their sequential setup differs from general continual learning, these design choices underscore that forgetting remains a practical concern even in RL-based pipelines.

Our proposed \method builds on these insights by introducing trigger-based supervision that activates SFT only on hard states, thereby combining the sample efficiency of expert demonstrations with the forgetting-resistant properties of on-policy RL.

\input{table/hybrid_related_compare.tex}

\newpage

\subsubsection{Hybrid SFT-RL Methods}

Recent methods combine supervised and reinforcement updates in several distinct ways. UFT~\citep{liu2025uft} optimizes a unified objective that applies dense supervision throughout training, targeting single-task reasoning performance and convergence speed. LUFFY~\citep{yan2025learning} mixes on-policy RL with off-policy trajectories through regularized importance sampling, emphasizing improved reasoning beyond the model's initial capability frontier. These methods are related to our work, but their design goals differ from continual agent training. In particular, they do not explicitly control the density of supervised updates using a state-dependent trigger derived from verified outcomes. Table~\ref{tab:hybrid_related} summarizes these differences.

\subsubsection{Comparison with Classical CL Methods}

Classical continual learning (CL) methods such as EWC~\citep{kirkpatrick2017overcoming}, LwF~\citep{li2017learning}, PackNet~\citep{mallya2018packnet}, and experience replay~\citep{rolnick2019experience} are designed for models with millions of parameters and assume access to prior task data or model snapshots. Scaling these methods to LLMs with billions of parameters poses significant challenges: EWC's Fisher matrix computation becomes intractable, LwF's distillation doubles memory requirements, and replay buffers grow linearly with task history.
In contrast, \method bridges this gap by combining RL's forgetting resistance with selective SFT, achieving both goals without auxiliary models, second-order statistics, or replay mechanisms.

\section{Additional Continual Learning Experiments}
\label{sec:additional_experiments}

The additional experiments in this appendix explore whether the same mechanism continues to matter if we move beyond that base setup. We therefore examine three complementary questions: whether TSRL still behaves well over a longer training chain with cumulative forgetting, whether its gains persist against stronger hybrid baselines that also mix RL and SFT, and whether the effect continues to hold on an additional model. Taken together, these results help separate the core benefit of trigger-based supervision from artifacts of a particular task pair or model choice.

\subsection{Longer Continual Learning Sequence with ScienceWorld}

To evaluate whether forgetting accumulates gracefully over longer training chains, we extend the continual-learning sequence from two stages to three stages: ALFWorld $\rightarrow$ WebShop $\rightarrow$ ScienceWorld. The model first acquires an ALFWorld prior, is then adapted to WebShop, and finally fine-tuned on ScienceWorld. We report performance on the final target task together with retention on both earlier tasks. The longer sequence confirms the same qualitative pattern as the two-task setting. Compared with the dense SFT$\to$RL baselines, \method$_{\text{GRPO}}$ improves retention on ALFWorld by 2.9 points and on WebShop by 2.6 points, while still slightly improving the final ScienceWorld result. The DAPO variant shows the same behavior, improving both prior-task metrics and the final ScienceWorld score over SFT$\to$DAPO. Pure RL remains the strongest retention-oriented baseline, but \method stays close to it while recovering the sample-efficiency benefits of selective expert guidance.

\input{table/continual_3task.tex}

\subsection{Additional Hybrid Baselines and Joint Training}

We further compare \method against representative hybrid baselines. \textbf{Iterative-SFT} alternates between online data collection and supervised updates, following the intuition that approximately on-policy supervised data can reduce forgetting~\citep{chen2025retaining}. \textbf{LUFFY} mixes on-policy RL with off-policy traces through regularized importance sampling~\citep{yan2025learning}. We also include \textbf{joint training} baselines that optimize a fixed weighted sum of RL and SFT losses, written as $\mathcal{L}_{\mathrm{RL}} + \alpha \mathcal{L}_{\mathrm{SFT}}$.

\input{table/hybrid_joint_baselines.tex}

Among the additional hybrid baselines, Iterative-SFT preserves prior-task performance reasonably well but is consistently weaker than TSRL on the new task, suggesting that alternating SFT updates alone is not sufficient to exploit expert knowledge efficiently in continual agent training. LUFFY achieves stronger target-task performance than GRPO, but still incurs more forgetting than TSRL, indicating that applying off-policy supervision to all samples remains too aggressive in a continual setting. The fixed-weight joint-training baselines clarify that the key benefit does not come from simply mixing two losses. Adding $0.1\mathcal{L}_{\mathrm{SFT}}$ or $0.5\mathcal{L}_{\mathrm{SFT}}$ to GRPO improves the WebShop target score relative to pure GRPO, but prior-task retention drops sharply and remains far below TSRL. The same pattern appears in the reverse direction, where joint training lifts ALFWorld acquisition but sacrifices substantially more WebShop retention than our trigger-based design.

\subsection{Performance on Qwen3-4B Model}

To test whether TSRL still maintains its continual learning capability on other models, we repeat the two-task continual-learning experiments using Qwen3-4B model~\citep{yang2025qwen3}.

\input{table/qwen3_4b.tex}

The results show that changing the base model does not eliminate forgetting. Relative to SFT$\to$GRPO, TSRL improves ALFWorld retention by 6.7 points in the WebShop direction and WebShop retention by 5.6 points in the reverse direction, while keeping target-task performance essentially unchanged. This indicates that the forgetting-mitigation effect of TSRL is not limited to the base model choice.

\section{Additional Discussions}
\label{sec:additional_discussions}

The empirical results show that TSRL works well in practice, but two conceptual points still deserve clarification. First, the method assumes that expert supervision can be accessed when repeated failures expose a capability gap, so we explain what ``querying the expert'' means in our experiments and in broader agent-training settings. Second, the trigger is defined through group success rate, which raises the question of how this signal relates to instance difficulty and policy stochasticity. The following discussion addresses these two issues and clarifies why TSRL uses an outcome-based trigger rather than an internal-confidence heuristic.

\subsection{On the Assumption of Querying the Expert}

TSRL is motivated by the practical setting where some form of expert supervision is available while catastrophic forgetting remains a concern. In our experiments, this supervision comes from offline buffers of successful trajectories, so ``querying the expert'' typically means retrieving the corresponding demonstration when the current policy repeatedly fails on an on-policy state. More broadly, the expert module is not restricted to a static offline teacher. In self-evolving settings, an agent can cache its own successful trajectories and later reuse them as supervision when the same capability degrades, effectively turning TSRL into a form of self-distillation. In online settings, the expert can instead be a stronger teacher model or a task-specific heuristic solver that is invoked when repeated failures reveal a capability gap. The assumption is sparse access to informative trajectories, rather than dense annotation for all states.

\subsection{Correlation Between GSR, Instance Difficulty, and Policy Stochasticity}

We view difficulty as an agent-relative concept rather than an absolute property of the instance. Policy stochasticity acts as a probe: if the current policy is noisy but fundamentally capable, sampling a group of trajectories will usually yield a non-zero GSR, which means the instance is still solvable through on-policy exploration. In contrast, when the GSR stays at zero (or below the trigger threshold) despite multiple samples, the failure is unlikely to be explained by stochastic variance alone. In that regime, the state is hard \emph{for the current policy}, and expert supervision becomes useful. This perspective also clarifies why GSR differs from confidence- or entropy-based signals. Internal uncertainty can be miscalibrated, especially in long-horizon agent tasks where models may be confidently wrong. GSR is instead outcome-based: it measures whether stochastic exploration actually discovers successful trajectories. As a result, it provides a more reliable trigger for deciding when to inject expert guidance.

\section{Experiment Details}

\subsection{Datasets}
\label{sec:data}
\paragraph{WebShop.}
WebShop~\citep{yao2022webshop} is a simulated e-commerce environment containing over 1.18 million real products and 12,087 crowd-sourced purchase instructions. Agents navigate the website through searching, browsing, and filtering to locate products that match user specifications. Upon selecting ``buy,'' the environment computes a continuous reward in $[0, 1]$ based on attribute-matching heuristics between the purchased item and the target product.

\paragraph{ALFWorld.}
ALFWorld~\citep{shridhar2021alfworld} provides interactive TextWorld environments paralleling the embodied ALFRED dataset~\citep{shridhar2020alfred}. Agents complete household tasks (e.g., ``Put a washed apple in the kitchen fridge'') through text-based interactions across six task types. The benchmark includes \textit{seen} and \textit{unseen} splits for evaluating in-distribution and out-of-distribution generalization, respectively. The environment returns sparse binary rewards indicating task completion.

\paragraph{Deep Search.}
Deep Search combines BrowseComp~\citep{wei2025browsecomp} and BrowseComp-ZH~\citep{zhou2025browsecomp}. BrowseComp comprises 1,266 challenging questions requiring persistent web browsing to locate hard-to-find, entangled information across numerous websites. BrowseComp-ZH extends this paradigm to the Chinese web with 289 multi-hop questions across 11 domains, natively constructed to address unique challenges such as fragmented platforms and implicit linguistic patterns. Both benchmarks provide sparse binary rewards based on exact answer matching.

\paragraph{ScienceWorld.}
ScienceWorld~\citep{wang-etal-2022-scienceworld} is a text-based interactive benchmark grounded in elementary-school science curricula. Agents must solve multi-step tasks spanning domains such as chemistry, thermodynamics, and electricity by interacting with a simulated environment. We use ScienceWorld as the third stage of a longer continual-learning sequence to test whether forgetting accumulates across multiple task transitions.

\paragraph{Math Reasoning.}
AIME 24 and AIME 25~\citep{numina_math_datasets} are derived from the American Invitational Mathematics Examination, a prestigious high school competition. Each exam features 15 challenging problems spanning algebra, number theory, combinatorics, and geometry, requiring multi-step reasoning with integer answers. We use these benchmarks to measure the capability retention of reasoning abilities after continual learning.

\subsection{Training Details}
\label{sec:training_details}
\paragraph{WebShop and ALFWorld.}
We implement our experiments using the Verl-Agent framework\footnote{\url{https://github.com/langfengQ/verl-agent}}~\citep{feng2025gigpo}, built upon the veRL codebase\footnote{\url{https://github.com/volcengine/verl}}~\citep{wan2025robust} (both under the Apache License 2.0). Verl-Agent offers a scalable infrastructure for multi-turn RL training of LLM agents, featuring flexible per-step input structures, customizable history management, and modular memory components. The framework supports diverse RL algorithms and agent environments, making it well-suited for our experiments. Table~\ref{tab:hyper_webshop} summarizes the key hyperparameters.
Unless otherwise stated, the additional 7B continual-learning experiments in Appendix~\ref{sec:additional_experiments} reuse the same training recipe, including the rollout group size, trigger threshold, and annealing schedule.

\paragraph{Deep Search.}
For the Deep Search, we equip the agent with two tools: Bing Search for web queries and a web viewer capable of reading page content and summarizing lengthy articles. The training configuration differs from the smaller-scale experiments to accommodate the increased complexity of multi-hop retrieval tasks, as provided in Table~\ref{tab:hyper_deepsearch}.

\subsection{Prompts}

We present the prompts used for WebShop, ALFWorld, and Deep Search in Figures~\ref{fig:webshop_prompt}, \ref{fig:alfworld_prompt}, and \ref{fig:deepsearch_prompt}, respectively.

\input{table/hyper_7b.tex}

\input{table/hyper_32b.tex}

\begin{figure}[ht]
\begin{tcolorbox}[
    colframe=black!80,
    colback=white,
    fonttitle=\bfseries\small,
    title=Instruction Prompt for WebShop,
    boxrule=0.5pt,
    arc=2pt,
    left=6pt,
    right=6pt,
    top=4pt,
    bottom=4pt
]
\small
You are an expert autonomous agent operating in the WebShop e-commerce environment.

\vspace{4pt}
Your task is to: {\{task\_description\}}.

\vspace{4pt}
Prior to this step, you have already taken {\{step\_count\}}
 step(s). Below are the most recent {\{history\_length\}} observations and the corresponding actions you took: {\{action\_history\}}.

\vspace{4pt}
You are now at step {\{current\_step\}} and your current observation is: {\{current\_observation\}}.

\vspace{4pt}
Your admissible actions of the current situation are: [{\{available\_actions\}}].

\vspace{4pt}
Now it's your turn to take one action for the current step.
You should first reason step-by-step about the current situation, then think carefully about which admissible action best advances the shopping goal. This reasoning process MUST be enclosed within <think> </think> tags.
Once you've finished your reasoning, you should choose an admissible action for current step and present it within <action> </action> tags.
\end{tcolorbox}
\caption{Prompt template for WebShop.}
\label{fig:webshop_prompt}
\end{figure}

\begin{figure}[ht]
\begin{tcolorbox}[
    colframe=black!80,
    colback=white,
    fonttitle=\bfseries\small,
    title=Instruction Prompt for ALFWorld,
    boxrule=0.5pt,
    arc=2pt,
    left=6pt,
    right=6pt,
    top=4pt,
    bottom=4pt
]
\small
You are an expert agent operating in the ALFRED Embodied Environment. Your task is to: {\{task\_description\}}

\vspace{4pt}
Prior to this step, you have already taken {\{step\_count\}} step(s). Below are the most recent {\{history\_length\}} observations and the corresponding actions you took: {\{action\_history\}}

\vspace{4pt}
You are now at step {\{current\_step\}} and your current observation is: {\{current\_observation\}}

\vspace{4pt}
Your admissible actions of the current situation are: [{\{admissible\_actions\}}].

\vspace{6pt}
Now it's your turn to take an action.

\vspace{4pt}
You should first reason step-by-step about the current situation. This reasoning process MUST be enclosed within [{<think> </think>}] tags.

\vspace{4pt}
Once you've finished your reasoning, you should choose an admissible action for current step and present it within {<action> </action>} tags.
\end{tcolorbox}
\caption{Prompt template for ALFWorld.}
\label{fig:alfworld_prompt}
\end{figure}

\begin{figure}[ht]
\begin{tcolorbox}[
    colframe=black!80,
    colback=white,
    fonttitle=\bfseries\small,
    title=Instruction Prompt for Deep Search,
    boxrule=0.5pt,
    arc=2pt,
    left=6pt,
    right=6pt,
    top=4pt,
    bottom=4pt
]
\small
Today is {\{current\_date\}}. You are a knowledgeable and information retrieval expert, very good at understanding and analyzing user questions, and providing rich and professional responses by combining your own knowledge and using the given tools.

\vspace{4pt}
You can call the following tools:[{\{"name": "search", "description": "Use a search engine to get information from the web.","parameters": [\{"name": "query", "description": "The search keyword for search engine", "type": "str"\}] , "required": ["query"]\}},
{\{"name": "open" , "description": "Open URL and return the web content related to info\_wanted", "parameters": [\{"name": "url" , "description": "The URL of the web page to open", "type": "str"\},\{"name": "info\_wanted", "description": "The description of the content to get", "type": "str"\}], "required": ["url", "info\_wanted"]\}}]

\vspace{4pt}
Please Note:

\vspace{4pt}
You must think first and put your thoughts in <think> </think>.

\vspace{4pt}
Please put the final answer in {<ans> </ans>} to ensure it is concise and direct.

\vspace{4pt}
Before finding the final answer, please try to find the answer in the current tool's capabilities as much as possible, do not give up.

\vspace{4pt}
The question is: {\{task\_description\}}.
\end{tcolorbox}
\caption{Prompt template for Deep Search.}
\label{fig:deepsearch_prompt}
\end{figure}

\end{document}

%% file: table/7b_result.tex
{
\addtolength{\extrarowheight}{\belowrulesep}

\aboverulesep=0pt
\belowrulesep=0pt
\begin{table*}[ht]
\centering
\caption{\textbf{Performance comparison on ALFWorld and WebShop benchmarks} (averaged over 3 random seeds). For ALFWorld, we report the average success rate (\%); for WebShop, we report both the average score and success rate (\%). TSRL consistently achieves performance improvements with both GRPO and DAPO backbones. The plus/minus signs (+/-) indicate the performance change compared to the corresponding SFT→RL baseline.}
\label{tab:7b_result}
\renewcommand{\arraystretch}{1}
\small

\begin{tabular}{>{}lccc<{}}
\toprule
\textbf{Method}                        & \textbf{ALFWorld}         & \textbf{WebShop (Score)}        & \textbf{WebShop (Succ.)}       \\ \midrule
\multicolumn{4}{l}{\color[HTML]{8B8280}\textit{\textbf{Base: Closed-Source Model}}}                                                             \\
\rowcolor[HTML]{EFEFEF}Prompting GPT-4o                                     & 48.0                      & 31.8                            & 23.7                           \\
Prompting  Gemini-2.5-Pro                            & 60.3                      & 42.5                            & 35.9                           \\ \midrule \midrule
\multicolumn{4}{l}{\color[HTML]{8B8280}\textit{\textbf{Base: Qwen2.5-7B-Instruct}}}                                       \\
\rowcolor[HTML]{EFEFEF}Prompting Qwen2.5                                    & 14.8                      & 26.4                            & 7.8                            \\
Prompting ReAct                                      & 31.2                      & 46.2                            & 19.5                           \\
\rowcolor[HTML]{EFEFEF}SFT                                                  & 59.3                      & 54.2                            & 28.6                           \\
GRPO                                                  & 74.4                      & 77.2                            & 64.3                           \\
\rowcolor[HTML]{EFEFEF}SFT$\to$GRPO                                & 76.6                      & 79.1                            & 65.8                           \\
DAPO                                                  & 86.8                      & 88.9                            & 73.3                           \\
\rowcolor[HTML]{EFEFEF}SFT$\to$DAPO                                & 87.0                      & 89.2                            & 74.5                           \\ \midrule
\rowcolor[HTML]{EBECFE}\textbf{\method$_\text{GRPO}$}                                                 & \textbf{80.9} \textcolor{gray}{(+4.3)}                   & \textbf{82.5} \textcolor{gray}{(+3.4)}                               & \textbf{69.3} \textcolor{gray}{(+3.5)}                              \\ \hdashline
\rowcolor[HTML]{EBECFE}\textbf{\method$_\text{DAPO}$}                                                 & \textbf{90.4} \textcolor{gray}{(+3.4)}                     & \textbf{90.7} \textcolor{gray}{(+1.5)}                               & \textbf{78.9} \textcolor{gray}{(+4.4)}                              \\ \bottomrule
\end{tabular}

\end{table*}}

%% file: table/continual_7b.tex
{
\addtolength{\extrarowheight}{\belowrulesep}
\aboverulesep=0pt
\belowrulesep=0pt

\begin{table*}[ht]
\centering

\caption{\textbf{Continual learning experiments between WebShop and ALFWorld.} This table reports performance on both \emph{prior} and new tasks when a model trained on one task is fine-tuned on another. \emph{Prior-task Model} denotes the LLM agent after GRPO training on the \emph{prior} task.}
\label{tab:continual_7b}
\renewcommand{\arraystretch}{1}

\resizebox{\textwidth}{!}{
\begin{tabular}{>{}lcccccc<{}}
\toprule

\multirow{2}{*}{{\textbf{Method}}}& \multicolumn{3}{c}{\textbf{Continual Learning in WebShop}}
& \multicolumn{3}{c}{\textbf{Continual Learning in ALFWorld}}
\\ \cmidrule(l){2-4} \cmidrule(l){5-7}
&\textit{\textbf{ALFWorld}} & \textbf{WebShop (Score)} & \textbf{WebShop (Succ.)} & \textbf{ALFWorld} & \textit{\textbf{WebShop (Score)}} & \textit{\textbf{WebShop (Succ.)}}
\\
\midrule

\rowcolor[HTML]{EFEFEF}\textbf{\textit{Prior-task Model}} & \textit{74.4} & 22.0 & 2.8
&  24.2 & \textit{77.2} & \textit{64.3} \\
SFT & \textit{69.3} & 49.1 & 18.3
 &  64.2 & \textit{58.9} & \textit{29.4} \\
\rowcolor[HTML]{EFEFEF}GRPO & \textit{74.2} & 86.7 & 70.0
 & 78.0 & \textit{79.3} & \textit{65.0} \\
 SFT$\to$GRPO & \textit{70.3} & 85.4 & 60.2
 & 84.3 & \textit{64.1} & \textit{58.7} \\
\rowcolor[HTML]{EFEFEF}DAPO & \textit{72.1} & 86.9 & 69.5
&  88.2 & \textit{79.1} & \textit{62.5} \\
SFT$\to$DAPO & \textit{73.4} & 87.2 & 70.4
&  90.0 & \textit{66.0} & \textit{63.1} \\ \midrule

\rowcolor[HTML]{EBECFE}
\textbf{\method$_\text{GRPO}$} & \textit{75.5 \textcolor{gray}{(+5.2)}} & \textbf{88.3} \textcolor{gray}{(+2.9)} & \textbf{74.5} \textcolor{gray}{(+14.3)}
& \textbf{85.2} \textcolor{gray}{(+0.9)} & \textit{81.6 \textcolor{gray}{(+17.5)}} & \textit{68.3 \textcolor{gray}{(+9.6)}} \\ \hdashline

\rowcolor[HTML]{EBECFE}
\textbf{\method$_\text{DAPO}$} & \textit{76.8 \textcolor{gray}{(+3.4)}} & \textbf{90.4} \textcolor{gray}{(+3.2)} & \textbf{75.2} \textcolor{gray}{(+4.8)}
 & \textbf{90.3} \textcolor{gray}{(+0.3)} & \textit{79.0 \textcolor{gray}{(+13.0)}} & \textit{64.6 \textcolor{gray}{(+1.5)}} \\

\bottomrule
\end{tabular}
}
\end{table*}
}

%% file: table/32b_result.tex
{
\addtolength{\extrarowheight}{\belowrulesep}

\aboverulesep=0pt
\belowrulesep=0pt
\begin{table*}[ht]
\centering
\caption{\textbf{Performance with Qwen2.5-32B-Instruct} on Deep Search (target training task) and Math Reasoning (measuring retention of base model capabilities).}
\label{tab:32b_result}
\renewcommand{\arraystretch}{1}

\small
\begin{tabular}{>{}lcccc<{}}
\toprule
\multirow{2}{*}{\textbf{Method}}             & \multicolumn{2}{c}{\textbf{Deep Search}} & \multicolumn{2}{c}{\textbf{Math Reasoning}} \\ \cmidrule(l){2-3} \cmidrule(l){4-5}
                            & \textbf{BrowseComp} & \textbf{BrowseComp-ZH} & \textbf{AIME 24}         & \textbf{AIME 25}        \\ \midrule
\multicolumn{5}{l}{\color[HTML]{8B8280}\textit{\textbf{Base: Closed-Source Model}}}          \\
\rowcolor[HTML]{EFEFEF}Prompting GPT-4o            & 0.7            & 6.9            & 9.3             & 13.1           \\
Prompting Claude-3.5-Sonnet & 0.8            & 5.9            & 16.0            & 7.4            \\ \midrule \midrule
\multicolumn{5}{l}{\color[HTML]{8B8280}\textit{\textbf{Base: Qwen2.5-32B-Instruct}}}         \\
\rowcolor[HTML]{EFEFEF}Prompting Qwen2.5           & 0.6            & 6.5            & \textbf{18.4}   & 11.5           \\
SFT                         & 3.8            & 19.4           & 11.3            & 10.0           \\
\rowcolor[HTML]{EFEFEF}RL                          & 6.7            & 28.8           & 13.3            & 16.7           \\
SFT$\to$RL        & \textbf{14.5}  & \textbf{35.2}  & 12.1            & 11.4           \\ \midrule
\rowcolor[HTML]{EBECFE}\textbf{\method}                        & 13.4 \textcolor{gray}{(-1.1)}           & 33.7 \textcolor{gray}{(-1.5)}          & 17.3 \textcolor{gray}{(+5.2)}           & \textbf{18.2} \textcolor{gray}{(+6.8)}  \\ \bottomrule
\end{tabular}

\end{table*}}

%% file: table/anneal-ablation.tex
{
\addtolength{\extrarowheight}{\belowrulesep}

\aboverulesep=0pt
\belowrulesep=0pt
\begin{table}[t]
\centering
\caption{\textbf{Ablation studies} on the annealing mechanism.}
\renewcommand{\arraystretch}{1.1}
\label{tab:anneal-abla}
\resizebox{\linewidth}{!}{
\begin{tabular}{>{}lccc<{}}
\toprule
\textbf{Method}           & \textbf{ALFWorld} & \textbf{WebShop (Score)} & \textbf{WebShop (Succ.)} \\ \midrule
\multicolumn{4}{l}{\color[HTML]{8B8280}\textit{\textbf{Constant Weight}}}         \\
\rowcolor[HTML]{EFEFEF} fixed 1.0                        & 63.5              & 72.9                     & 52.8                     \\
fixed 0.1                        & 69.8              & 75.1                     & 57.2                     \\ \midrule
\multicolumn{4}{l}{\color[HTML]{8B8280}\textit{\textbf{Annealing With Different Annealing Steps}}} \\
\rowcolor[HTML]{EFEFEF}25 steps               & 76.4              & 78.8                     & 61.2                     \\
50 steps               & 78.2              & 79.3                     & 60.6                     \\
\rowcolor[HTML]{EFEFEF} 100 steps (Ours) & \textbf{80.9}              & \textbf{82.5}                     & \textbf{69.3}                     \\ \bottomrule
\end{tabular}
}
\vspace{-1em}
\end{table}}

%% file: table/hybrid_related_compare.tex
{
\addtolength{\extrarowheight}{\belowrulesep}
\aboverulesep=0pt
\belowrulesep=0pt

\begin{table*}[t]
\centering
\caption{\textbf{Positioning \method among representative hybrid SFT-RL methods.}}
\label{tab:hybrid_related}
\renewcommand{\arraystretch}{1.2}
\resizebox{\textwidth}{!}{
\begin{tabular}{>{}lccc<{}}
\toprule
\textbf{Method} & \textbf{Primary Goal} & \textbf{How Supervision Enters} & \textbf{Forgetting Control} \\
\midrule
\rowcolor[HTML]{EFEFEF}UFT~\citep{liu2025uft} & Single-task post-training & Dense unified SFT+RL objective & No explicit continual-learning control \\
LUFFY~\citep{yan2025learning} & Reasoning beyond prior capability & Off-policy traces mixed with on-policy RL & Partial, but not continual-specific \\
\rowcolor[HTML]{EFEFEF}Iterative-SFT~\citep{chen2025retaining} & Improve with on-policy data & Alternating collection and updates & Milder drift, no explicit trigger \\
\textbf{\method} & Continual agent training & Sparse state-dependent trigger + annealing & Explicitly designed to reduce forgetting \\
\bottomrule
\end{tabular}
}
\end{table*}
}

%% file: table/continual_3task.tex
{
\addtolength{\extrarowheight}{\belowrulesep}
\aboverulesep=0pt
\belowrulesep=0pt

\begin{table*}[t]
\centering
\caption{\textbf{Three-stage continual learning experiment: ALFWorld $\rightarrow$ WebShop $\rightarrow$ ScienceWorld.} \emph{Prior-task Model} denotes the model after the first two stages, before adaptation to ScienceWorld.}
\label{tab:continual_3task}
\renewcommand{\arraystretch}{1}
\resizebox{0.8\textwidth}{!}{
\begin{tabular}{>{}lccc<{}}
\toprule
\textbf{Method} & \textit{\textbf{ALFWorld (Prior)}} & \textit{\textbf{WebShop (Prior)}} & \textbf{ScienceWorld (Continual)} \\
\midrule
\rowcolor[HTML]{EFEFEF}\textbf{\textit{Prior-task Model}} & \textit{74.2} & \textit{70.0} & 8.5 \\
SFT & \textit{69.0} & \textit{65.4} & 25.2 \\
\rowcolor[HTML]{EFEFEF}GRPO & \textit{71.7} & \textit{68.9} & 29.3 \\
SFT$\to$GRPO & \textit{68.6} & \textit{65.7} & 31.8 \\
\rowcolor[HTML]{EFEFEF}DAPO & \textit{71.3} & \textit{68.9} & 29.5 \\
SFT$\to$DAPO & \textit{68.5} & \textit{65.4} & 31.4 \\ \midrule
\rowcolor[HTML]{EBECFE}\textbf{\method$_\text{GRPO}$} & \textit{71.5 \textcolor{gray}{(+2.9)}} & \textit{68.3 \textcolor{gray}{(+2.6)}} & \textbf{31.9} \textcolor{gray}{(+0.1)} \\ \hdashline
\rowcolor[HTML]{EBECFE}\textbf{\method$_\text{DAPO}$} & \textit{69.7 \textcolor{gray}{(+1.2)}} & \textit{66.3 \textcolor{gray}{(+0.9)}} & \textbf{31.8} \textcolor{gray}{(+0.4)} \\
\bottomrule
\end{tabular}
}
\end{table*}
}

%% file: table/hybrid_joint_baselines.tex
{
\addtolength{\extrarowheight}{\belowrulesep}
\aboverulesep=0pt
\belowrulesep=0pt

\begin{table*}[ht]
\centering
\caption{\textbf{Comparisons with additional hybrid baselines and fixed-weight joint training.} We report success rates for both continual-learning directions.}
\label{tab:hybrid_joint}
\renewcommand{\arraystretch}{1}
\resizebox{\textwidth}{!}{
\begin{tabular}{>{}lcccc<{}}
\toprule
\multirow{2}{*}{\textbf{Method}} & \multicolumn{2}{c}{\textbf{ALFWorld $\rightarrow$ WebShop}} & \multicolumn{2}{c}{\textbf{WebShop $\rightarrow$ ALFWorld}} \\
\cmidrule(l){2-3} \cmidrule(l){4-5}
& \textit{\textbf{ALFWorld (Prior)}} & \textbf{WebShop (Continual)} & \textbf{ALFWorld (Continual)} & \textit{\textbf{WebShop (Prior)}} \\
\midrule
\rowcolor[HTML]{EFEFEF}\textbf{\textit{Prior-task Model}} & \textit{74.4} & 2.8 & 24.2 & \textit{64.3} \\
SFT & \textit{69.3} & 18.3 & 64.2 & \textit{29.4} \\
\rowcolor[HTML]{EFEFEF}GRPO & \textit{74.2} & 70.0 & 78.0 & \textit{65.0} \\
SFT$\to$GRPO & \textit{70.3} & 60.2 & 84.3 & \textit{58.7} \\
\rowcolor[HTML]{EFEFEF}Iterative-SFT~\cite{chen2025retaining} & \textit{74.8} & 69.8 & 82.5 & \textit{64.9} \\
LUFFY~\cite{yan2025learning} & \textit{74.0} & 73.6 & 85.0 & \textit{64.8} \\
\rowcolor[HTML]{EFEFEF}$\mathcal{L}_{\mathrm{RL}} + 0.1 \mathcal{L}_{\mathrm{SFT}}$ & \textit{70.5} & 71.3 & 82.3 & \textit{58.7} \\
$\mathcal{L}_{\mathrm{RL}} + 0.5 \mathcal{L}_{\mathrm{SFT}}$ & \textit{70.1} & 71.2 & 82.6 & \textit{58.2} \\ \midrule
\rowcolor[HTML]{EBECFE}\textbf{\method$_\text{GRPO}$} & \textbf{\textit{75.5}} & \textbf{74.5} & \textbf{85.2} & \textbf{\textit{68.3}} \\
\bottomrule
\end{tabular}
}
\end{table*}
}

%% file: table/qwen3_4b.tex
{
\addtolength{\extrarowheight}{\belowrulesep}
\aboverulesep=0pt
\belowrulesep=0pt

\begin{table*}[ht]
\centering
\caption{\textbf{Continual learning results with Qwen3-4B.} We report success rates for both continual-learning directions using the GRPO backbone.}
\label{tab:qwen3_4b}
\renewcommand{\arraystretch}{1}
\resizebox{\textwidth}{!}{
\begin{tabular}{>{}lcccc<{}}
\toprule
\multirow{2}{*}{\textbf{Method}} & \multicolumn{2}{c}{\textbf{ALFWorld $\rightarrow$ WebShop}} & \multicolumn{2}{c}{\textbf{WebShop $\rightarrow$ ALFWorld}} \\
\cmidrule(l){2-3} \cmidrule(l){4-5}
& \textit{\textbf{ALFWorld (Prior)}} & \textbf{WebShop (Continual)} & \textbf{ALFWorld (Continual)} & \textit{\textbf{WebShop (Prior)}} \\
\midrule
\rowcolor[HTML]{EFEFEF}\textbf{\textit{Prior-task Model}} & \textit{84.5} & 15.3 & 23.7 & \textit{54.9} \\
SFT & \textit{77.3} & 62.5 & 66.4 & \textit{25.8} \\
\rowcolor[HTML]{EFEFEF}GRPO & \textit{83.8} & 74.2 & 77.6 & \textit{49.7} \\
SFT$\to$GRPO & \textit{77.5} & 79.0 & 82.5 & \textit{44.3} \\ \midrule
\rowcolor[HTML]{EBECFE}\textbf{\method$_\text{GRPO}$} & \textbf{\textit{84.2}} \textcolor{gray}{(+6.7)} & 78.8 \textcolor{gray}{(-0.2)} & \textbf{83.0} \textcolor{gray}{(+0.5)} & \textbf{\textit{49.9}} \textcolor{gray}{(+5.6)} \\
\bottomrule
\end{tabular}
}
\end{table*}
}

%% file: table/hyper_7b.tex
{
\addtolength{\extrarowheight}{\belowrulesep}
\aboverulesep=0pt
\belowrulesep=0pt
\begin{table}[h]
\small
\caption{WebShop and ALFWorld's Hyperparameters.}
\label{tab:hyper_webshop}
\centering
\renewcommand{\arraystretch}{1}
\begin{tabular}{>{}lc<{}}
\toprule
\textbf{Parameter} & \textbf{Value} \\
\midrule
\multicolumn{2}{l}{\color[HTML]{8B8280}\textbf{\textit{Optimization}}} \\
\rowcolor[HTML]{EFEFEF}Actor Learning Rate & 1e-6 \\
Batch Size & 64 \\
\rowcolor[HTML]{EFEFEF}Training Steps & 150 \\
\midrule
\multicolumn{2}{l}{\color[HTML]{8B8280}\textbf{\textit{Regularization}}} \\
\rowcolor[HTML]{EFEFEF}KL Loss Coefficient & 0.01 \\
Entropy Coefficient & 0.001 \\
\midrule
\multicolumn{2}{l}{\color[HTML]{8B8280}\textbf{\textit{Policy Clipping}}} \\
\rowcolor[HTML]{EFEFEF}Clip Range (GRPO) & 0.2 \\
Clip Low / High (DAPO) & 0.2 / 0.28 \\
\midrule
\multicolumn{2}{l}{\color[HTML]{8B8280}\textbf{\textit{Rollout Configuration}}} \\
\rowcolor[HTML]{EFEFEF}Group Size & 8 \\
Temperature & 1.0 \\
\rowcolor[HTML]{EFEFEF}Max Turns (ALFWorld) & 50 \\
Max Turns (WebShop) & 15 \\
\midrule
\multicolumn{2}{l}{\color[HTML]{8B8280}\textbf{\textit{Trigger-based Supervision}}} \\
\rowcolor[HTML]{EFEFEF}Trigger Threshold & 0 \\
Annealing Step & 100 \\
\bottomrule
\end{tabular}
\end{table}}

%% file: table/hyper_32b.tex
{\addtolength{\extrarowheight}{\belowrulesep}
\aboverulesep=0pt
\belowrulesep=0pt
\begin{table}[h]
\small
\caption{Deep Search's Hyperparameters.}
\label{tab:hyper_deepsearch}
\centering
\renewcommand{\arraystretch}{1}
\begin{tabular}{>{}lc<{}}
\toprule
\textbf{Parameter} & \textbf{Value} \\
\midrule
\multicolumn{2}{l}{\color[HTML]{8B8280}\textbf{\textit{Optimization}}} \\
\rowcolor[HTML]{EFEFEF} Actor Learning Rate & 1e-6 \\
Batch Size & 256 \\
\rowcolor[HTML]{EFEFEF}Training Steps & 100 \\
\midrule
\multicolumn{2}{l}{\color[HTML]{8B8280}\textbf{\textit{Regularization}}} \\
\rowcolor[HTML]{EFEFEF} KL Loss Coefficient & 0.001 \\
Entropy Coefficient & 0.0 \\
\midrule
\multicolumn{2}{l}{\color[HTML]{8B8280}\textbf{\textit{Policy Clipping}}} \\
\rowcolor[HTML]{EFEFEF}Clip Range & 0.2 \\
\midrule
\multicolumn{2}{l}{\color[HTML]{8B8280}\textbf{\textit{Rollout Configuration}}} \\
\rowcolor[HTML]{EFEFEF}Group Size & 8 \\
Temperature & 1.0 \\
\rowcolor[HTML]{EFEFEF}Max Turns & 80 \\
\midrule
\multicolumn{2}{l}{\color[HTML]{8B8280}\textbf{\textit{Trigger-based Supervision}}} \\
\rowcolor[HTML]{EFEFEF}Trigger Threshold & 0 \\
Annealing Step & 50 \\
\bottomrule
\end{tabular}
\vspace{-1em}
\end{table}}